\documentclass[runningheads]{llncs}
\usepackage{graphicx}
\usepackage{amsmath}
\usepackage{algorithm}
\usepackage{algorithmic}
\usepackage{latexsym,amssymb,amscd}
\usepackage{color,xcolor}
\usepackage{placeins}
\newcommand{\citep}{\cite}
\newcommand{\dd}{\mathcal{\dagger}}

\begin{document}
\title{Wasserstein Proximal of GANs}
%
%
\author{Alex Tong Lin\inst{1} \and
Wuchen Li\inst{2} \and
Stanley Osher\inst{1} \and
Guido Mont\'{u}far\inst{1,3}
}
\authorrunning{Lin, Li, Osher, Mont\'{u}far}
%
\institute{University of California, Los Angeles, Los Angeles CA 90095, USA \and
University of South Carolina, Columbia SC 29208, USA \and
Max Planck Institute for Mathematics in Sciences, Leipzig 04103, Germany}
\maketitle              
\begin{abstract}
We introduce a new method for training generative adversarial networks by applying the Wasserstein-2 metric proximal on the generators. The approach is based on Wasserstein information geometry. It defines a parametrization invariant natural gradient by pulling back optimal transport structures from probability space to parameter space. We obtain easy-to-implement iterative regularizers for the parameter updates of implicit deep generative models. 
Our experiments demonstrate that this method improves the speed and stability of training in terms of wall-clock time and Fr\'echet Inception Distance 

\end{abstract}

\section{Introduction}

Generative Adversarial Networks (GANs) \citep{OGAN} are a powerful approach to learning generative models. Here, a discriminator tries to tell apart the data generated by a real source and the data generated by a generator, whereas the generator tries to fool the discriminator. This adversarial game is formulated as an optimization problem over the discriminator and an implicit generative model for the generator. 
An implicit generative model is a parametrized family of functions mapping a noise source to sample space. In trying to fool the discriminator, the generator should try to recreate the real source. 

The problem of recreating a target density can be formulated as the minimization of a discrepancy measure. The Kullback--Leibler (KL) divergence is known to be difficult to work with when the densities have a low dimensional support set, as is commonly the case in applications with structured data and high dimensional sample spaces. An alternative  is to use the Wasserstein distance or Earth Mover's distance, which is based on optimal transport theory. This has been used recently to define the loss function for learning generative models~\citep{Boltzman,LWL}. In particular, the Wasserstein GAN~\citep{WGAN} has attracted much interest in recent years. 

Besides defining the loss function, optimal transport can also be used to introduce structures serving the \emph{optimization} itself, in terms of the gradient operator. In full probability space, this method is known as the Wasserstein steepest descent flow \cite{JKO,otto2001}. 
In this paper we derive the Wasserstein steepest descent flow for deep generative models in GANs. 
We use the Wasserstein-2 metric function, which allows us to obtain a Riemannian structure and a corresponding natural (i.e., Riemannian) gradient. 
A well known example of a natural gradient is the Fisher-Rao natural gradient, which is induced by the KL-divergence. 
In learning problems, one often finds that the natural gradients offer advantages compared to the Euclidean gradient~\citep{NG,IG}. 

In GANs, the densities under consideration typically have a small support set, which prevents implementations of the Fisher-Rao natural gradient. 
Therefore, we propose to use the gradient operator induced by the Wasserstein-2 metric on probability models~\citep{LM,LiMontufar2018_riccia}.

We propose to compute the parameter updates of the generators in GANs by means of a proximal operator where the proximal penalty
is a squared constrained Wasserstein-2 distance. 
In practice, the constrained distance can be approximated by a 
neural network. 
In implicit generative models, the constrained Wasserstein-2 metric exhibits a simple structure. 
We generalize the Riemannian metric and introduce two methods: the relaxed proximal operator for generators and the semi-backward Euler method. 
Both approaches lead to practical numerical implementations 
of the Wasserstein proximal operator for GANs. The method can be easily implemented as a drop-in regularizer for the generator updates. Experiments demonstrate that this method improves the stability of training and reduces the training time. 

\smallskip 

This paper is organized as follows. 
In Section~\ref{section2} we introduce the Wasserstein natural gradient and proximal optimization methods. 
In Section~\ref{sec:implicit} we review basics of implicit generative models. 
In Section~\ref{Compute} we derive practical computational methods and study their theoretical properties. 
In Section~\ref{section4} we demonstrate the effectiveness of the proposed methods in experiments with various types of GANs. 
In Section~\ref{headings} we comment on related work, and in Section~\ref{sec:discussion} we offer a brief discussion.

\section{Wasserstein natural proximal optimization}\label{section2}

In this section, we present the Wasserstein natural gradient and the corresponding proximal method. 

\subsection{Motivation and illustration}
\label{ex}

The natural gradient method is an approach to parameter optimization in probability models, which has been promoted especially within information geometry~\citep{IG,IG2}. 
This method chooses the steepest descent direction when the size of the step is measured by means of a metric on probability space. 

In this way, the natural gradient is parameterization invariant \citep{NG} and provides more stability in training. 
In contrast, the ordinary gradient method follows the steepest descent direction calculated from Euclidean distance in parameter space. This can be unstable because distances in parameter space do not reflect distances in probability space, and the parameterization of the model affects the descent direction.

If $F(\theta)$ is the loss function, 
the steepest descent direction 
is the vector $d\theta$ that solves 
\begin{align}\label{nat-grad_fish-rao}
\min_{d\theta} F(\theta + d\theta)\quad\textrm{subject to}\quad D(\rho_{\theta}, \rho_{\theta +d\theta}) = \epsilon, 
\end{align}
for a small enough $\epsilon$. 
Here $D$ is a divergence function on probability space. 
Expanding the divergence to second order and solving 
leads to an update of the form 
\begin{equation*}
d\theta \;\propto\; G(\theta)^{-1}\nabla_\theta F(\theta), 
\end{equation*}
where $G$ is the Hessian 
of $D$. 
Usually the Fisher-Rao metric is considered for $G$, which corresponds to having $D$ as the KL-divergence. 

In this work, we use structures derived from optimal transport. Concretely, we replace $D$ in equation (\ref{nat-grad_fish-rao}) with the Wasserstein-$p$ distance. This is defined as 
\begin{equation}
W_p(\rho_\theta, \rho_{\theta^k})^p=\inf \int_{\mathbb{R}^n\times \mathbb{R}^n}\|x-y\|^p \pi(x,y)dxdy, 
\label{eq:W}
\end{equation}
where the infimum is over all joint probability densities $\pi(x,y)$ with marginals $\rho_{\theta}$, $\rho_{\theta^k}$. 
We will focus on $p=2$. The Wasserstein-2 distance introduces a metric tensor in probability space, making it an infinite dimensional Riemannian manifold. We will introduce a finite dimensional metric tensor $G$ on the parameter space of a generative model. 

The Wasserstein metric allows us to define a natural gradient even when the support of the distributions is low dimensional and the Fisher-Rao natural gradient is not well defined. We will use the proximal operator, which computes the parameter update by minimizing the loss function plus a penalty on the step size. 
This saves us the need to compute the matrix $G$ and its inverse explicitly. 
As we will show, the Wasserstein metric can be translated to practical proximal methods for implicit generative models. 
We first present a toy example, with explicit calculations, to illustrate the effectiveness of Wasserstein proximal operator. 

\begin{figure}[h]
	\centering
	\includegraphics[width=5cm]{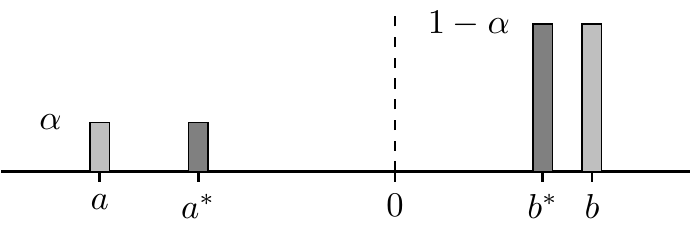}
	\includegraphics[width=4.5cm]{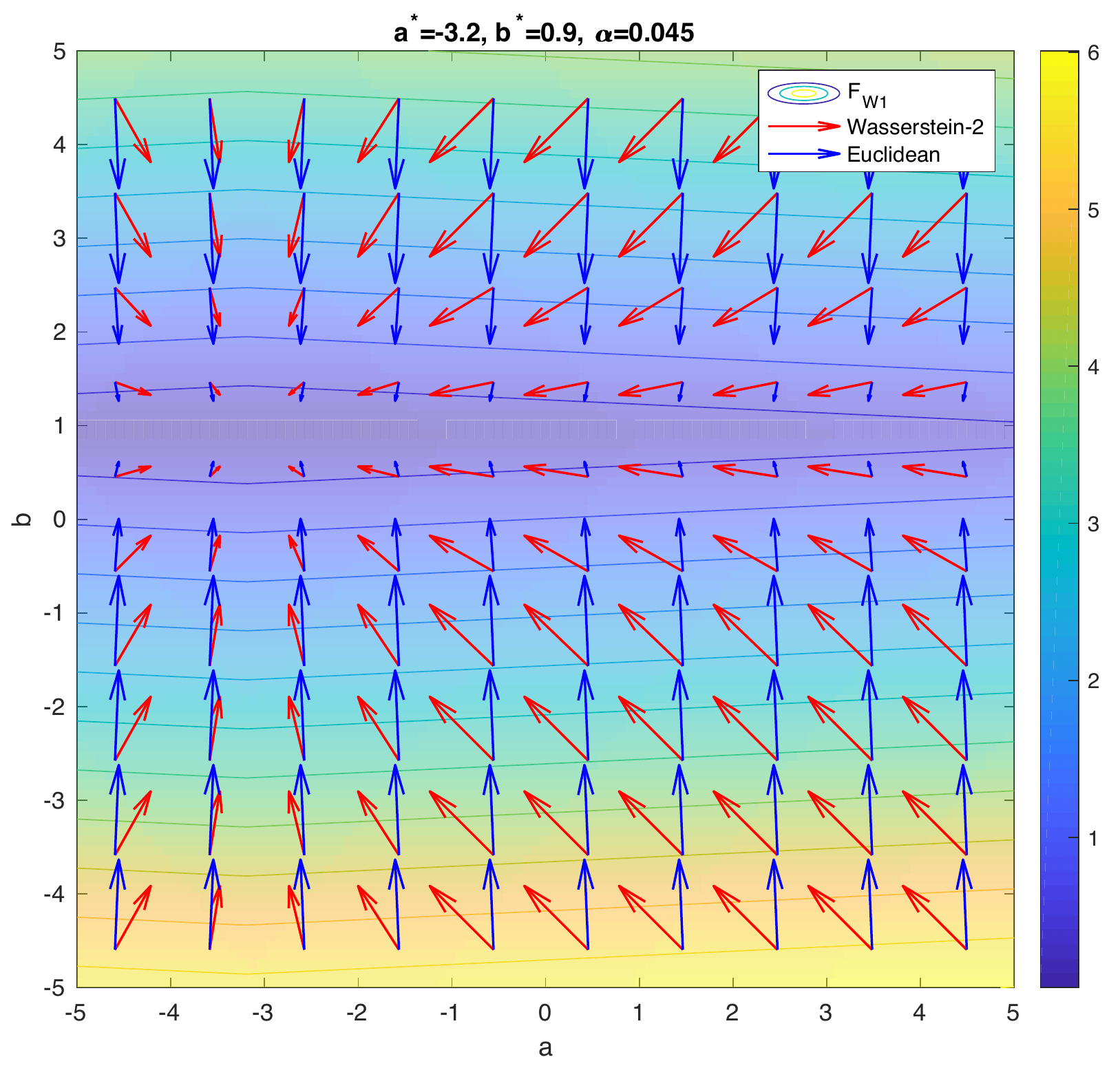}
	\caption{Illustration of the Wasserstein proximal operator. Here the Wasserstein proximal penalizes parameter steps in proportion to the mass being transported, which results in updates pointing towards the minimum of the loss function. The Euclidean proximal penalizes all parameters equally, which results in updates naively orthogonal to the level sets of the loss function.}
	\label{figure1}
\end{figure}

\begin{example}\label{exp1}
	Consider a probability model consisting of mixtures of pairs of delta measures. 
	Let $\Theta=\{\theta=(a,b)\in\mathbb{R}^2 \colon a<0<b\}$, and define
	\begin{equation*}
	\rho(\theta, x)=\alpha \delta_{a}(x)+(1-\alpha) \delta_b(x),
	\end{equation*}
	where $\alpha \in [0,1]$ is a given ratio and $\delta_a(x)$ is the delta measure supported at point $a$. See Figure~\ref{figure1}. 
	For a loss function $F$, 
	the proximal update is 
	\begin{equation*}
	\theta^{k+1}=\arg\min_{\theta\in\Theta} F(\theta)+\frac{1}{2h}D(\rho_\theta, \rho_{\theta^k}). 
	\end{equation*}
	We check the following common choices for the function $D$ to measure the distance between $\theta$ and $\theta^k$, $\theta\neq \theta^k$.
	
	\begin{enumerate}
		\item Wasserstein-2 distance:
		\begin{equation*}
		W_2(\rho_{\theta}, \rho_{\theta^k})^2= \alpha (a-a^k)^2+(1-\alpha)(b-b^k)^2;
		\end{equation*}
		\item Euclidean distance:
		\begin{equation*}
		\|\theta-\theta^k\|^2= {(a-a^k)^2+(b-b^k)^2};
		\end{equation*}
		\item Kullback--Leibler divergence:
		\begin{equation*}
		D_{\operatorname{KL}}(\rho_\theta\|\rho_{\theta^k})=\int_{\mathbb{R}^n} \rho(\theta,x) \log\frac{\rho(\theta,x)}{\rho(\theta^k,x)}dx=\infty;
		\end{equation*}
		\item $L^2$-distance:
		\begin{equation*}
		{L^2}(\rho_\theta, \rho_{\theta^k})=\int_{\mathbb{R}^n}|\rho(\theta, x)-\rho(\theta^k,x)|^2dx =\infty.
		\end{equation*}
	\end{enumerate}
	As we see, the KL-divergence and $L^2$-distance take value infinity, which tells the two parameters apart, but does not quantify the difference in a useful way. 
	The Wasserstein-2 and Euclidean distances still work in this case. 
	The Euclidean distance captures the difference in the locations of the delta measures, but not their relative weights. On the other hand, the Wasserstein-2 takes these into account. 
	The right panel of Figure~\ref{figure1} illustrates the loss function $F(\theta) = W_1(\rho_\theta, \rho_{\theta^\ast})$ for a random choice of $\theta^\ast$, alongside with the Euclidean and Wasserstein-2 proximal parameter updates. 
	The Wasserstein proximal update points more consistently in the direction of the global minimum. 
\end{example}

\subsection{Wasserstein natural gradient}

We next present the Wasserstein natural gradient operator for parametrized probability models. 

\begin{definition}[Wasserstein natural gradient operator]\label{prop2} 
	Given a model of probability densities 
	$\rho(\theta, x)$ over $x\in\mathbb{R}^n$, with locally injective parametrization by $\theta\in \Theta\subseteq\mathbb{R}^d$, and a loss function $F\colon \Theta\rightarrow \mathbb{R}$, 
	the Wasserstein natural gradient operator 
	is given by 
	\begin{equation*}
	\operatorname{grad} F(\theta)=G(\theta)^{-1}\nabla_\theta F(\theta). 
	\end{equation*}
	Here $G(\theta)=(G(\theta)_{ij})_{1\leq i,j\leq d}\in\mathbb{R}^{d\times d}$ is the matrix with entries 
	\begin{equation*}
	G(\theta)_{ij}=\int_{\mathbb{R}^n}\Big(\nabla_{\theta_i}\rho(\theta,x), \mathcal{G}(\rho_\theta)\nabla_{\theta_j}\rho(\theta,x)\Big) dx,
	\end{equation*}
	where $\mathcal{G}(\rho_\theta)$ is the Wasserstein-2 metric tensor in probability space. More precisely, $\mathcal{G}(\rho)=(-\Delta_{\rho})^{-1}$ is the inverse of the elliptic operator $\Delta_{\rho}:=\nabla\cdot(\rho\nabla)$.
\end{definition}

For completeness, we briefly explain the definition of the Wasserstein natural gradient.  The gradient operator on a Riemannian manifold $(\Theta, g)$ is defined as follows. 
For any $\sigma\in T_\theta\Theta$, the Riemannian gradient $\nabla_\theta^g F(\theta)\in T_\theta \Theta$ satisfies $g_\theta(\sigma, \textrm{grad} F(\theta))= (\nabla_\theta F(\theta), \sigma)$. 
In other words, $\sigma^\top G(\theta)  \textrm{grad}F(\theta)=\nabla_\theta F(\theta)^\top \sigma$. 
Since $\theta\in \mathbb{R}^d$ and $G(\theta)$ is positive definite, $\textrm{grad}F(\theta)=G(\theta)^{-1}\nabla_\theta F(\theta)$. 

Our main focus will be in deriving practical computational methods that allow us to apply these structures to optimization in GANs. 
Consider the gradient flow of the loss function: 
\begin{equation}\label{GD}
\frac{d\theta}{dt}=-\textrm{grad} F(\theta)=-G(\theta)^{-1}\nabla_\theta F(\theta). 
\end{equation}
There are several
discretization schemes for a gradient flow of this type. 
One of them is the forward Euler method, known as the 
steepest descent method: 
\begin{equation}\label{FE}
\theta^{k+1}=\theta^k-hG(\theta^k)^{-1} \nabla_\theta F({\theta^k}), 
\end{equation}
where $h>0$ is the learning rate (step size). 
In practice we usually do not have a closed formula for the metric tensor $G(\theta)$. 
In~\eqref{FE}, we need to solve for the inverse Laplacian operator, the Jacobian of the probability model, and
compute the inverse of $G(\theta)$.
When the parameter $\theta\in \Theta$ is high dimensional, these computations are impractical. 
Therefore, we will consider a different approach based on the proximal method. 

\subsection{Wasserstein natural proximal}

To practically apply the Wasserstein natural gradient, we present an alternative way to
discretize the gradient flow, known as the proximal method {or backward Euler method}. The proximal operator computes updates of the form 
\begin{equation}\label{proximal}
\theta^{k+1}=\arg\min_{\theta}~F(\theta)+\frac{\textrm{Dist}(\theta, \theta^k)^2}{2h},
\end{equation}
where $\textrm{Dist}$ is an iterative regularization term, given by the Riemannian distance: 
\begin{equation*}
\textrm{Dist}(\theta, \theta^k)^2=  \inf\Big\{\int_0^1 \dot\theta_t^\top G(\theta_t)\dot\theta_t dt\colon \theta_0=\theta,~\theta_1=\theta^k\Big\}.
\end{equation*}
Here the infimum is taken among all continuously differentiable parameter paths $\theta_t=\theta(t)\in \Theta$, $t\in[0,1]$. 
The proximal operator is defined implicitly, in terms of a minimization problem, but in some cases it can be written explicitly. 
Interestingly, it allows us to consider an iterative regularization term in the parameter update. 

We observe that there are two time variables in the proximal update (\ref{proximal}). 
One is the time discretization of gradient flow, known as the learning rate $h>0$; 
the other is the time variable in the definition of the Riemannian distance $\textrm{Dist}(\theta, \theta^k)$. 
The variation in the time variable of the Riemannian distance can be further simplified as follows. 

\begin{proposition}[Semi-backward Euler method]
	\label{prop3}
	The iteration 
	\begin{equation}\label{new}
	\theta^{k+1}=\arg\min_{\theta} F(\theta)+\frac{\tilde{D}(\theta, \theta^k)^2}{2h},
	\end{equation}
	with 
	\begin{equation*}
	\tilde{D}(\theta, \theta^k)^2=\int_{\mathbb{R}^n}\Big( \rho_\theta-\rho_{\theta^k} , \mathcal{G}(\rho_{\tilde\theta})(\rho_\theta-\rho_{\theta^k})\Big)dx,
	\end{equation*}
	and $\tilde\theta=\frac{\theta+\theta^k}{2}$, is a consistent time discretization of the Wassserstein natural gradient flow~(\ref{GD}).
\end{proposition}

Here the distance term in~(\ref{proximal}) is replaced by $\tilde{D}$, which is obtained by a mid-point approximation in time. 
The mid-point $\tilde \theta$ can be chosen in many ways between $\theta$ and $\theta^k$. 
For simplicity and symmetry,
we let $\tilde\theta=\frac{\theta+\theta^k}{2}$. 
In practice, we also use $\tilde\theta = \theta^k$, since in this case $\mathcal{G}(\rho_{\tilde\theta})$ can be held fixed when iterating over $\theta$ to obtain~\eqref{new}. 
Formula~(\ref{new}) is called the semi-backward Euler method (SBE), because it can also be expressed as 
\begin{equation*}
\theta^{k+1}=\theta^k-hG(\tilde\theta)^{-1}\nabla_\theta F(\theta^{k+1})+o(h).   
\end{equation*}
The proof is contained in the appendix.

We point out that all methods described above, i.e., the forward Euler method~(\ref{FE}), 
the backward Euler method~(\ref{proximal}), 
and the semi-backward Euler method~(\ref{new}), 
are time consistent discretizations of the Wasserstein natural gradient flow~(\ref{GD}) with first order accuracy in time. 
We shall focus on the semi-backward Euler method and derive practical formulas for the iterative regularization term. 

\section{Implicit generative models} 
\label{sec:implicit}

Before proceeding, we briefly recall the setting of Generative Adversarial Networks (GANs). The practical purpose of GANs is to train a model to produce samples from a (complicated) target distribution.
This technique has been met with remarkable success today. 

GANs consist of two parts: the \emph{generator} and the \emph{discriminator}. The generator is a function $g_\theta\colon \mathbb{R}^\ell \rightarrow \mathbb{R}^n$ that takes inputs $z$ in latent space $\mathbb{R}^\ell$ with distribution $p(z)$ (a common choice is a Gaussian) to outputs $x=g_\theta(z)$ in sample space $\mathbb{R}^n$ with distribution $\rho(\theta,x)$. 
The objective of training is to find a value of the parameter $\theta$ so that $\rho(\theta,x)$ matches a given target distribution, say $\rho_{\text{target}}(x)$. 
The discriminator is merely an assistance during optimization of a GAN in order to obtain the right parameter value for the generator. 
It is a function $f_\omega:\mathbb{R}^n \rightarrow \mathbb{R}$, whose role is to discriminate real images (sampled from the target distribution) from fake images (produced by the generator). 

To train a GAN, one works on min-maxing a function such as
\begin{equation*}
\inf_{\theta} \sup_{\omega} \mathbb{E}_{x\sim \rho_{\text{target}}(x)}\left[ \log f_\omega(x) \right] + \mathbb{E}_{z\sim p(z)}\left[ \log(1 - f_\omega(g_\theta(z)) \right]. 
\end{equation*}
The specific loss function can be chosen in many different ways (including the Wasserstein-1 loss~\citep{IMGAN,WGAN}), but the above is the one that was first considered for GANs, and is a common choice in applications. 
The first term is interpreted as the log of the confidence that the discriminator has about the data $x$ being genuine,
and the second term is interpreted as the log of the confidence that the discriminator has about the data $g_\theta(z)$ being not genuine. 
During training, we ideally want to train the discriminator to detect real-world samples from generator samples, but then at the end of training, we want a generator that produces samples that are indistinguishable from real samples (and thus will also fool the discriminator). 

Practically, to perform the optimization, we adopt an alternating gradient optimization scheme for the generator parameter $\theta$ and the discriminator parameter $\omega$. 
This is iterated until a sufficient convergent criteria is reached (usually examining when the loss functions stabilize). We will implement a Wasserstein proximal method for optimizing GANs. 

\section{Computational methods}
\label{Compute}

In this section, we present two methods for implementing the 
Wasserstein natural proximal for GANs. 
The first method is based on solving the variational formulation of the proximal penalty over an affine space of functions. 
This leads to a low-order version of the Wasserstein metric tensor $\mathcal{G}(\rho_\theta)$. 
The second method is based on a formula for the Wasserstein metric tensor for 1-dimensional sample spaces, which we relax to sample spaces of arbitrary dimension. 

\subsection{Affine space variational approximation}
\label{computationa:SBElinear}

The mid point approximation $\tilde{D}$ from Proposition~\ref{prop3} can be written using dual coordinates (cotangent space) of probability space in the variational form 
\begin{equation*}
\tilde{D}(\theta, \theta^k)^2
=\sup_{\Phi\in C^{\infty}(\mathbb{R}^n)}\Big\{\int_{\mathbb{R}^n}\Phi(x)(\rho(\theta, x)-\rho(\theta^k,x))  - \frac{1}{2}\|\nabla\Phi(x)\|^2\rho(\tilde\theta,x) \, dx\Big\}.
\end{equation*}
In order to obtain an explicit formula, 
we consider a function approximator of the form 
\begin{equation*}
\Phi_\xi(x) := \sum_{j}\xi_j \psi_j(x)=\xi^\top\Psi(x),
\end{equation*}
where $\Psi(x)=(\psi_j(x))_{j=1}^K$ are given basis functions on sample space $\mathbb{R}^n$, and $\xi=(\xi_j)_{j=1}^K\in\mathbb{R}^K$ is the parameter. 
In other words, we consider 
\begin{equation}
\begin{split}
\tilde{D}(\theta, \theta^k)^2=&\sup_{\xi\in \mathbb{R}^K} \Big\{\int_{\mathbb{R}^n}\Phi_\xi(x)(\rho(\theta, x)-\rho(\theta^k,x)) - \frac{1}{2}\|\nabla\Phi_\xi(x)\|^2\rho(\tilde\theta,x) \, dx\Big\}.
\end{split}
\label{eq:affinedef}
\end{equation}

\begin{theorem}[Affine metric function $\tilde D$]\label{thm4}
	Consider some $\Psi=(\psi_1,\ldots,\psi_K)^{\top}$ and 
	assume that $M(\theta)=(M_{ij}(\theta))_{1\leq i,j\leq K}\in \mathbb{R}^{K\times K}$ is a regular matrix with entries
	\begin{equation*}
	M_{ij}(\theta)=\mathbb{E}_{Z\sim p}\Big(\sum_{l=1}^n\partial_{x_l}\psi_i(g(\tilde\theta, Z)) \partial_{x_l}\psi_j(g(\tilde\theta, Z))  \Big),  
	\end{equation*}
	where $\tilde \theta=\frac{\theta+\theta^k}{2}$. 
	Then, 
	\begin{equation*}
	\begin{split}
	\tilde{D}(\theta, \theta^k)^2 = &\Big(\mathbb{E}_{Z\sim p}[\Psi(g(\theta, Z))-\Psi(g(\theta^k, Z))]\Big)^\top \\
	&M(\tilde \theta)^{-1}\Big(\mathbb{E}_{Z\sim p}[\Psi(g(\theta, Z))-\Psi(g(\theta^k, Z))]\Big). 
	\end{split}
	\end{equation*}
\end{theorem}

The proof is contained in the appendix. There are many possible choices for the basis $\Psi$. 

For example, if $K=n$ and $\psi_k(x)=x_k$, $k=1,\ldots, n$, then $M(\theta)$ is the identity matrix. 
In this case, 
\begin{equation*}
\tilde{D}(\theta, \tilde\theta)^2=\|\mathbb{E}_{Z\sim p}(g(\theta, Z)-g(\theta^k,Z))\|^2.
\end{equation*}
We will focus on degree one and degree two polynomials. The algorithms are presented in Section~\ref{subsec:algorithms}. {We note we experimented with a three neural network version for Wasserstein natural gradient, where we used an additional neural network to approximate $\Phi(x)$. However, the additional neural network was computationally burdensome in computing the gradient direction. So practically we stick with the above affine approximation with current two neural networks. We notice that the three network version will approximate the Wasserstein natural gradient accurately for scientific computing problems; see details in \cite{2020arXiv200211309L}.}

\subsection{Relaxation from 1-D}
\label{computationa:RWP}

Now we present a second method for approximating $\tilde{D}$. In the case of implicit generative models with 1-dimensional sample space, the constrained Wasserstein-2 metric tensor has an explicit formula. 
This allows us to define a relaxed Wasserstein metric for implicit generative models with sample spaces of arbitrary dimension. In dimension 1, we have (of which the proof is in the appendix)

\begin{theorem}[1-D sample space]\label{1d}
	If $n=1$, then
	\begin{multline*}
	\operatorname{Dist}(\theta_0, \theta_1)^2 = \inf \Big\{\int_0^1\mathbb{E}_{Z\sim p}\|\frac{d}{dt}g(\theta(t), Z)\|^2\, dt \colon \theta(0)=\theta_0, \theta(1)=\theta_1 \Big\}, 
	\end{multline*}
	where the infimum is taken over all continuously differentiable parameter paths. 
	Therefore, we have 
	\begin{equation*}
	\tilde{D}(\theta, \theta^k)^2=\mathbb{E}_{Z\sim p}\|g(\theta, Z)-g(\theta^k, Z)\|^2.
	\end{equation*}
\end{theorem} 

In sample spaces of dimension higher than one, 
we no longer have the explicit formula for $\tilde D$. 
The relaxed metric consists of using the same formulas from the theorem. 
Later on, we show that this formulation of $\tilde D$ still provides a metric with parameterization invariant properties in the proximal update. 

\subsection{Algorithms}
\label{subsec:algorithms}

The Wasserstein natural proximal method for GANs optimizes the parameter $\theta$ of the generator by the proximal iteration~(\ref{new}). 
We implement this in the following ways: 

\paragraph{\bf RWP method.} 
The first and simplest method follows Section~\ref{computationa:RWP}, and updates the generator by:
\begin{align*}
\theta^{k+1}=\arg\min_{\theta\in \Theta} F(\theta) +\frac{1}{2h}{\mathbb{E}_{Z\sim p}\|g(\theta,Z)-g(\theta^{k},Z)\|^2}. 
\end{align*}
We call this the Relaxed Wasserstein Proximal (RWP) method. 

\paragraph{\bf SBE order 1 method.} 
The second method is based on the discussion from Section~\ref{computationa:SBElinear}, approximating $\Phi$ by linear functions. 
We update the generator by: 
\begin{align*}
\theta^{k+1}=\arg\min_{\theta\in \Theta} F(\theta) +\frac{1}{2h}{\|\mathbb{E}_{Z\sim p}\Big( g(\theta,Z)-g(\theta^{k},Z)\Big)\|^2},
\end{align*}
We call this the Order-1 SBE (O1-SBE) method. 

\medskip 
\noindent
\emph{Derivation of SBE order 1.} 
Here $\psi_j(x)=x_j$. Thus if $i=j$, then denoting $x= g(\theta, Z)$ with $Z\sim p$,
\begin{equation*}
M_{ij}(\theta) = \mathbb{E}_{Z\sim p}\sum_{l=1}^n\partial_{x_l}\psi_i(x)\partial_{x_l}\psi_j(x) = 1. 
\end{equation*}
Otherwise, $M_{ij}(\theta)=0$, if $i\neq j$. Thus 
\begin{equation*}
M_{ij}(\tilde\theta)=
\begin{cases}
1 & \textrm{if $i=j$;}\\
0 &\textrm{otherwise.}
\end{cases}
\end{equation*}
This proves the result. 
\qed

\paragraph{\bf SBE order 2 method.}
In an analogous way to the SBE order 1 method, we can approximate $\Phi$ by quadratic functions,  
to obtain the Order-2 SBE (O2Diag-SBE) method: 
\begin{align*}
\theta^{k+1}=&\arg\min_{\theta\in \Theta} F(\theta)\\
&+\frac{1}{h}\bigg(\frac{1}{2}\|\mathbb{E}_{Z\sim p}[g(\theta,Z)-g(\theta^k, Z)]- \mathbb{E}_{Z\sim p}[Q g(\theta^k, Z)]\|^2 \\
&+ \frac{1}{2}\mathbb{E}_{Z\sim p}[\left<g(\theta, Z), Qg(\theta, Z)\right>] -\frac{1}{2}\mathbb{E}_{Z\sim p}[\left<g(\theta^k, Z),Qg(\theta^k, Z)\right>] \\
&- \frac{1}{2}\mathbb{E}_{Z\sim p}[\|Qg(\theta^k,Z)\|^2] \bigg), 
\end{align*}
where $Q=\textrm{diag}(q_i)_{i=1}^n$ is the diagonal matrix with diagonal entries 
\begin{align*}
q_i = &\frac{1}{2}\frac{\mathbb{E}_{Z\sim p}[(g(\theta, Z)_i-g(\theta^k, Z)_i)^2]}{\text{Var}(g(\theta^k, Z)_i)} + \frac{\text{Cov}_{Z\sim p}(g(\theta,Z)_i, g(\theta^k,Z)_i)}{\text{Var}_{Z\sim p}(g(\theta^k,Z)_i)} - 1 ,
\end{align*}
where $g(\theta, Z)_i$ is the $i$th coordinate of the samples. 

\medskip 
\noindent 
\emph{Derivation of SBE order 2.} 
\label{sec:O2SBE} 

Consider 
$$\Phi(x) = \frac{1}{2} x^\top Q x + a^\top x + b,$$ 
with a diagonal matrix $Q = \text{diag}(q_1, \ldots, q_N)$. 
We get that
\begin{equation*}
\begin{split}
&\sup_{a,Q} \Phi(g(\theta,z)) - \Phi(g(\theta_{k-1},(z)) - \frac{1}{2} \|\nabla \Phi(g({\theta_{k-1}}, Z))\|^2 \\
=&\frac{1}{2}\|\mathbb{E}[g(\theta,z)-g(\theta_{k-1},z)-Qg(\theta_{k-1},z)]\|^2  \\
&+\frac{1}{2}\mathbb{E}[g(\theta, Z)^\top Qg(\theta,Z)] - \frac{1}{2}\mathbb{E}[g(\theta_{k-1},Z)^\top Qg(\theta_{k-1},Z)] \\
&-\frac{1}{2}\mathbb{E}[\|Qg(\theta_{k-1},Z)\|^2]
\end{split}
\end{equation*}
which will be used in the O2Diag-SBE update. 
We note that $x=g(\theta,z)$ and $y=g(\theta_{k-1},z)$. Then we have that the above becomes
\begin{multline*}
\sup_{a,Q} \mathbb{E}_{x,y} \left[ a^\top(x-y) - \frac{1}{2}x^\top Qx - \frac{1}{2}y^\top  Qy - \|a + Qy\|^2 \right] \\
= \mathbb{E}_{x,y} \bigg[ a^\top(x-y) - \frac{1}{2}\sum q_i x_i^2 - \frac{1}{2}q_i y_i^2 - \|a + \text{diag}(q_1, \ldots, q_N) y\|^2 \bigg]. 
\end{multline*}
The above is a quadratic equation in $a$ and $Q = \text{diag}(q_1, \ldots, q_N)$, so we can formulate it as  
\begin{align}\label{eq:ell}
(a,Q)\ell - \frac{1}{2}(a,Q)M(a,Q)^\top, 
\end{align}
where $\ell = \Big(\mathbb{E}(x-y), \frac{1}{2}\mathbb{E}(x^2 - y^2)\Big)$, and where
\begin{align}\label{eq:M_diag_order_2}
M = \frac{1}{B}\sum_{b=1}^B \begin{pmatrix} 1 \\ y_b \end{pmatrix} \begin{pmatrix} 1 \\ y_b \end{pmatrix}^\top,  
\end{align}
which is the matrix for the quadratic term $\|a+\text{diag}(q_1, \ldots, q_N) y\|^2$. Then the maximum is attained at 
\begin{align*}
(a^\ast , Q^\ast ) = M^{-1} \ell . 
\end{align*}
By explicitly calculating the inverse $M^{-1}$ (where 
$B$ is a sufficiently large batch size such that $M$ is full rank) and multiplying $\ell$, 
we obtain the formula for $Q^\ast $.
\qed

\medskip 

The methods described above can be regarded as iterative regularizers. 
RWP 
penalizes the expected squared norm of the differences between samples (second moment differences). 
O1-SBE 
penalizes the squared norm of the expected differences between samples. 
O2Diag-SBE penalizes a combination of squared norm of the expected differences plus variances. 
They all encode statistical information of the generators. 
All these approaches \emph{regularize the generator} by the expectation and variance of the samples.
The implementation is shown in Algorithm~\ref{alg:genprox}. 
We also provide a detailed practical guide in Appendix~\ref{sec:rwp_detailed}. 
In the next subsection, we discuss the convergence and consistency properties of these methods. 

\begin{algorithm}
	\caption{Wasserstein Natural Proximal}
	\label{alg:genprox}
	\begin{algorithmic}[1]
		\REQUIRE ${F}_\omega$, a parameterized function to minimize (e.g., the Wasserstein-1 with a parameterized discriminator); $g_\theta$, the generator. 
		\REQUIRE $\text{Optimizer}_{{F}_\omega}$; $\text{Optimizer}_{g_\theta}$. 
		\REQUIRE $h$ proximal step-size; $B$ mini-batch size; max iterations; generator iterations. 
		\FOR{$k=0$ {\bfseries to} max iterations}
		\STATE{Sample real data $\{x_i\}_{i=1}^B$ and latent data $\{z_i\}_{i=1}^B$}
		\STATE{$\omega^k \leftarrow \text{Optimizer}_{{F}_\omega}\left( \frac{1}{B} \sum_{i=1}^B {F}_\omega(g_\theta(z_i)) \right)$}
		\FOR{$\ell=0$ {\bfseries to} generator iterations}
		\STATE{Sample latent data $\{z_i\}_{i=1}^B$}
		\STATE{$\tilde{D} = \text{RWP, or O1-SBE, or O2Diag-SBE (Sec.~\ref{subsec:algorithms})}$}
		\STATE{
			$\theta^k \leftarrow \text{Optimizer}_{g_\theta}\left(\frac{1}{B} \sum_{i=1}^B {F}_\omega(g_\theta(z_i)) + \frac{1}{2h}\tilde{D}(\theta, {\theta^k})^2\right)$.}
		\ENDFOR
		\ENDFOR
	\end{algorithmic}
\end{algorithm}

\FloatBarrier

\subsection{Theoretical guarantees}

We show that the Wasserstein natural proximal algorithms introduced in the previous sections are consistent. 

\begin{theorem}\label{thm6}
	Algorithm~\ref{alg:genprox} provides a consistent numerical time-discretization of the gradient flow
	\begin{equation*}
	\frac{d}{dt}\theta=-\tilde{G}(\theta)^{\dd}\nabla_\theta F(\theta).
	\end{equation*}
	Here $\tilde G^\dd$ is the pseudoinverse of the Hessian of $\tilde D$ and is a positive semi-definite matrix. 
	In particular, the loss function is a Lyapunov function of gradient flow, meaning that it is non-increasing along the gradient flow. If $\theta^\ast $ is a critical point of $F$ and $\lambda_{\min}\Big(\tilde{G}(\theta^\ast )^{\dd}\operatorname{Hess}F(\theta^\ast )\Big)> 0$,  
	then $\theta(t)$ locally converges to $\theta^\ast $. 
\end{theorem}

\begin{remark}
	{For specially selected families, our approximation of the metric are generalizations of closed form solutions for classical Wasserstein-2 distances. Here we present two examples for our two approximation methods.
		
		Firstly, in one dimensional sample space, consider example \ref{exp1}. Here the model becomes, 
		\begin{equation*}
		\rho(\theta, x)=\alpha \delta_{a}(x)+(1-\alpha) \delta_b(x),\quad \textrm{where}\quad \theta=(a,b).
		\end{equation*}
		In this case, our model is two dimensional, in which our metric in Theorem \ref{1d} is a constant matrix, e.g.  
		\begin{equation*}
		G_W(\theta)=\begin{pmatrix} 
		\alpha & 0\\
		0  & 1-\alpha
		\end{pmatrix}.
		\end{equation*}
		Hence the distance is 
		\begin{equation*}
		\textrm{Dist}(\theta_0,\theta_1)^2=\alpha\|a_0-a_1\|^2+(1-\alpha)\|b_0-b_1\|^2=\tilde{D}(\theta_0,\theta_1)^2.
		\end{equation*}
		Secondly, we can consider a Gaussian model, where $\rho(\theta, x)$ lies in a Gaussian distribution. After some direct calculations, we can observe that the order two affine approximation of the metric is also exact. }
\end{remark}
\begin{remark}
	{
		In general, the proposed metric is not exactly the Wassertein-2 metric within probability models. This is because the potential $\Phi$ needs to be solved by the constrained continuity equation, i.e. 
		\begin{equation}\label{ss}
		\nabla_\theta\rho(\theta, x)=-\nabla\cdot(\rho(\theta,x)\nabla\Phi).
		\end{equation}
		This equation can also be written into a weak form in generative models. The more accurate approximation methods for solving equation \eqref{ss} are left for future works. For example, \cite{2020arXiv200211309L} applies the another neural network for approximating \eqref{ss} and further solves the related gradient flows. }
\end{remark}

\begin{proof}[Proof of Theorem \ref{thm6}]
	Here we only present the second order expansion of $\tilde D$. 
	By Taylor expansion, we simply check that 
	\begin{equation*}
	\tilde{D}(\theta, \theta+h)=h^\top G(\theta) h +o(h^2), 
	\end{equation*}
	where
	\begin{equation*}
	\begin{split}
	\tilde{G}(\theta)_{ij}=&
	\Big\langle\mathbb{E}_{Z\sim p}\Psi(g(\theta,Z))\nabla_{\theta_i}g(\theta, z), M(\theta)\mathbb{E}_{Z\sim p}\Psi(g(\theta,Z))\nabla_{\theta_j}g(\theta, Z)\Big\rangle, 
	\end{split}
	\end{equation*}
	which is positive semi-definite. 
	Similar as the proof of Proposition~\ref{prop3}, we know that the algorithm has the update 
	\begin{equation*}
	\theta^{k+1}=\theta^k-hG(\tilde\theta)^{\dd}\nabla_\theta F(\theta^{k+1}) +o(h).
	\end{equation*}
	This is the first order time discretization of the gradient flow. 
	We next check that 
	\begin{equation*}
	\frac{d}{dt}F(\theta(t))=-\nabla_\theta F(\theta)^\top \tilde G(\theta)^{\dd} \nabla_\theta F(\theta)\leq 0. 
	\end{equation*}
	We observe that $F(\theta)$ decreases along the gradient flow. 
	This finishes the proof. 
	\qed
\end{proof}

Theorem~\ref{thm6} implies that the Wasserstein natural proximal methods that we developed in the previous sections, have the expected properties of natural (Riemannian) gradient flows, including parametrization invariance. 
We note that with the approximation, $G$ might not always be strictly positive definite, possibly introducing more critical points to the flow. 
This is a general phenomenon in gradient optimization with approximation and can be addressed by a variety of simple methods, such as the Levenberg-Marquard modification~\citep{chong2013introduction}, which simply adds $\lambda I$ with some $\lambda>0$. 

The Wasserstein metric in probability models 
can lead to different convergence rates and convergence regions than the Euclidean metric. 
Here the convergence region depends on the constrained Wasserstein metric within probability models. We will demonstrate the advantages of the method in the following experiments. 

\section{Experiments}
\label{section4}
\label{others}

We present numerical experiments evaluating the Relaxed Wasserstein Proximal (RWP) and Semi-Backward Euler (SBE) methods in the optimization of GANs. 
We find that our methods provide both better speed (measured by wallclock) and stability compared to regular gradient methods. 

\subsection{Experimental setup}
\label{resultsrwp}

The RWP, O1-SBE, and O2Diag-SBE algorithms are intended to be an easy-to-implement, drop-in replacement to improve speed and convergence of GAN training. 
These methods apply regularization on the generator updates during training. This stands in contrast to most GAN training methods, which regularize the discriminator, e.g., by a gradient penalty \citep{IMGAN,RWGAN,kodali2018on,2018arXiv180606621A,miyato2018spectral}. 
There has been limited exploration in regularizing the generator \citep{NIPS2016_6399}. 

Following Line 7 of Algorithm~\ref{alg:genprox}, for each update of the discriminator we update the generator $\ell$ times by 
\begin{align*}
\theta \leftarrow \text{Optimizer}_{\theta} \bigg( \text{Original loss} +  \frac{1}{2h} \tilde{D}(\theta,\theta^k) \bigg), 
\end{align*}
where $\tilde{D}(\theta, \theta^k)$ is one of the distances from Section~\ref{subsec:algorithms}. 
Here two hyperparameters are introduced: the proximal step-size $h$, and the number of iterations $\ell$. 
One may update the discriminator a number of times and then update the generator a number of times, and repeat; we call one loop of this update an {\em outer-iteration}. 
A more detailed description of the algorithm is given in Appendix~\ref{sec:rwp_detailed}. 
We test our methods 
on three types of GAN: 
Vanilla GANs \citep{OGAN} (Jenson-Shannon), 
WGAN-GP \citep{IMGAN}, 
and DRAGAN \citep{kodali2018on}.

\paragraph{\bf Neural network architectures and hyperparameter settings.} 
We utilize the DCGAN \citep{DBLP:journals/corr/RadfordMC15} architecture for the discriminator and generator. Specifically, the discriminator has convolutional and batch-norm layers with LeakyReLU activations, and a sigmoid output activation. The generator uses deconvolutional and batch-norm layers with ReLU activations, with a tanh output activation. 
Since we are testing our method as a drop-in regularizer, the hyperparameters (excluding $h$ and $\ell$) are chosen to work well before applying our regularization. 
The specific values that we used are provided in Appendix~\ref{sec:hp_for_rwp}. 

\paragraph{\bf Datasets.} 
We use the CIFAR-10 dataset \citep{Krizhevsky09learningmultiple}, and the aligned and cropped CelebA dataset \citep{liu2015faceattributes}. The CIFAR-10 dataset consists of 60,000 full-color images of size $32\times 32$. Each image belongs to one of ten classes: airplane, automobile, bird, cat, deer, dog, frog, horse, ship, truck. The aligned and cropped CelebA dataset contains 202,599 images of (Western) celebrity faces, and they are aligned so that all the faces are in the same position, and they are cropped to size $64\times 64$ in our experiments. 

\paragraph{\bf Quality measure.} 
Measuring the quality of a generative model for natural images is an open problem. 
Several methods have been proposed. The current state of the art method is the Fr\'echet Inception Distance (FID) \citep{NIPS2017_7240}. 
The FID computes the distance between two image distributions using Google's pre-trained Inception-v3 network. 
It examines the difference in activations of the pool3 layer of one image distribution compared to another. 
More precisely, it computes the mean and variance of the 2048-dimensional pool3 layer activations for a batch of real images and a batch of generated images, 
and then computes the Fr\'echet Distance (also known as the Wasserstein distance for Gaussians) for these means and variances. 
We note that when optimizing GANs, we are not directly optimizing the FID, and it is merely a secondary/after-the-fact measure for the quality of samples from the generator. 
We employ the FID both to measure performance and to measure convergence of GAN training (lower FID is better); we use 10,000 generated images to measure the FID. For CIFAR-10, we measure the FID every 1000 outer-iterations. 

\paragraph{\bf Latent space walk.} 
\cite{DBLP:journals/corr/RadfordMC15} suggest that walking in the latent space of an implicit generative model could given an indication of how well the model is doing at generalizing the training data. 
A latent space walk consists of sampling two points from the latent space, $z_1$ and $z_2$, and then generating images from the linear interpolation of these points. If the generator is generalizing well, then we should observe a gradual transition between images.

\paragraph{\bf Time to convergence.} 
Since our methods 
perform multiple generator iterations for each discriminator iteration, we compare against other methods not in terms of iterations, but in terms of 
wallclock time (this procedure was also used by \citep{NIPS2017_7240}).

\subsection{Results on the CIFAR-10 dataset}

Figure \ref{fig:cifar-10_comparison} shows that our regularizers 
improve the speed of convergence on CIFAR-10. 
In the case of DRAGAN, 
our regularizers greatly improve stability in the sense of less oscillations in FID values, and achieves lower FID values. 
In the case of WGAN-GP our SBE methods can reduce the FID about six times faster than the regular gradient. 
Overall, we found that the fastest method to train CIFAR-10 was Vanilla GAN with O1-SBE or RWP. 
For O2Diag-SBE, we obtained excellent performance without trying many different hyperparameter values. A different choice of $h$ and $\ell$ values might improve wallclock time.  

\begin{figure*}[h]
	\centering
	\includegraphics[width=.32\textwidth]{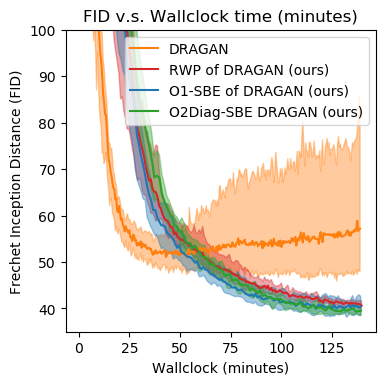}
	\includegraphics[width=.32\textwidth]{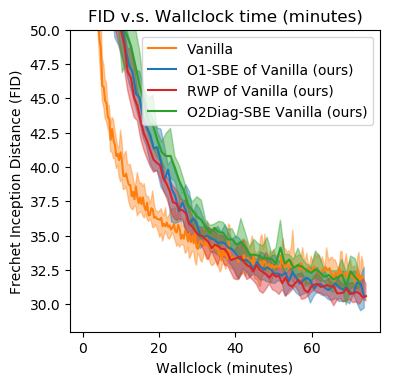}
	\includegraphics[width=.32\textwidth]{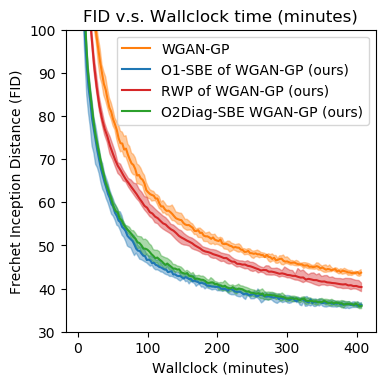}
	\caption{The effect of using RWP, O1-SBE, and O2Diag-SBE regularization on the CIFAR-10 dataset. The experiments are averaged over 5 runs. The bold lines are the average, and the enveloping lines are the minimum and maximum. From the three graphs, we see that using the easy-to-implement RWP, O1-SBE, O2Diag-SBE regularizations all improve speed as measured by wallclock time, and it also can achieve a lower FID. 
	}
	\label{fig:cifar-10_comparison}
\end{figure*}

In the appendix, Figure~\ref{app:fig:64_sample_cfiar10_wgangp_rwp} shows samples generated from WGAN-GP with RWP regularization, trained on the CIFAR-10 dataset. The FID for these images is 38.3. 
We also performed latent space walks \citep{DBLP:journals/corr/RadfordMC15} to show RWP regularization does not cause the GAN to memorize. 
In the appendix Figure~\ref{app:fig:cifar10_wgangp_rwp} we see that the images obtained from such trajectories have smooth transitions, indicating that GANs with RWP regularization generalize well. 
Order 1 SBE, and Order 2 Diagonal SBE showed similar results. 

\FloatBarrier

\subsection{Results on the CelebA dataset}

The top row of Figure~\ref{fig:celeba_standard_rwp} shows our results on the CelebA dataset. 
For this dataset we only examine 
the Vanilla and WGAN-GP GANs, as these are the two most popular frameworks. 
For Vanilla GANs, we see that RWP, O1-SBE, and O2Diag-SBE improve the speed of GAN training according to wallclock time, and they also achieve a slightly lower FID. 
In the case of WGAN-GP, adding our regularizers does not improve nor harm time or performance. 
Overall, the fastest method to train CelebA 
is Vanilla GAN with O2Diag-SBE regularization.

\begin{figure}[h]
	\centering
	\includegraphics[width=4cm]{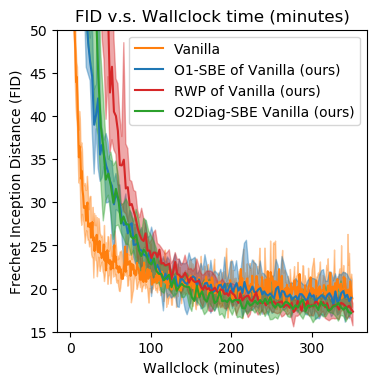}
	\includegraphics[width=4cm]{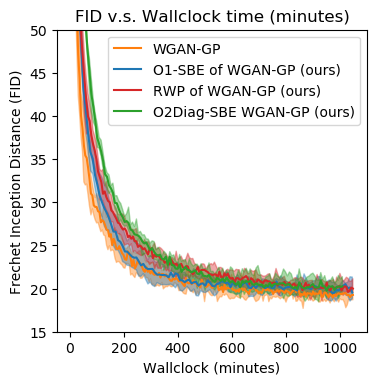}
	
	\includegraphics[width=4cm]{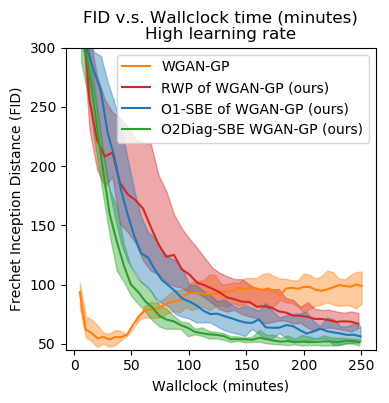}
	\includegraphics[width=4cm]{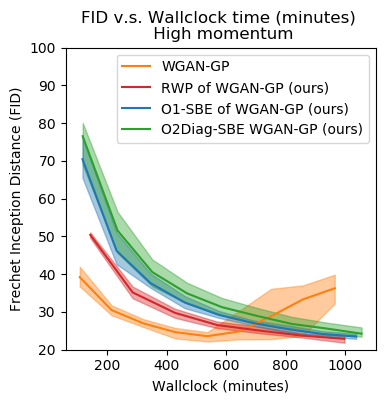}
	\includegraphics[width=4cm]{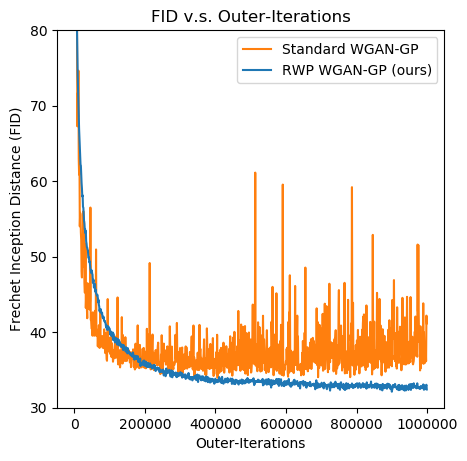} 
	
	\caption{
		\textbf{(Top Row):} The effect of using RWP, O1-SBE, and O2Diag-SBE regularization for training Vanilla GANs and WGAN-GP on the CelebA dataset. The experiment was averaged over 5 runs. 
		The bold lines are the average, and the enveloping lines are the minimum and maximum. 
		\textbf{(Bottom Row):} The \textbf{left} panel shows the FID values over training time for a high learning rate. The \textbf{middle} panel shows the FID values over training time for a high momentum value. RWP, O1-SBE, and O2Diag-SBE improve training by providing a lower FID when the learning rate or momentum is high. In the \textbf{right} panel, we show experiment demonstrating the effect of performing 10 generator iterations per outer-iteration with and without RWP. With RWP regularization we obtain convergence and a lower FID.  Without RWP, training is highly variable and the FID is even on a rising trend in the end.
	}
	\label{fig:celeba_standard_rwp}
	\label{fig:celeba_wgangp_rwp}
\end{figure}

In the appendix, Figure~\ref{app:fig:64_sample_celeba} shows samples generated from a Vanilla GAN with RWP regularization, trained on the CelebA dataset. The FID here is 17.105. 
In the appendix, Figure~\ref{app:fig:latent_space_walk} shows latent space walks.

\subsection{Stability}

In the bottom row of Figure \ref{fig:celeba_standard_rwp}, we see that adding our regularizers actually improves the stability of WGAN-GP under higher learning rates ($0.002$ vs.~$0.0001$, 20 times larger) and higher momentum (Adam $\beta_1=0.5$ vs.~$0$). For a high learning rate, we see without regularization, WGAN-GP first reaches an FID of around 50, but then destabilizes to an FID of around 100. For a higher momentum, the behavior is similar in that without regularization, the FID first decreases, but then increases. When applying regularization, the values decrease in a more monotonic fashion and tend to stay low. 

In the right of the bottom row of Figure \ref{fig:celeba_standard_rwp}, we optimize with 10 generator iterations per outer-iteration with and without RWP regularization. 
Without regularization, the FID varies much more wildly (and even starts rising near the end), 
but with RWP regularization, the FID values are stable. 

\FloatBarrier

\section{Related works}
\label{headings}

\paragraph{Wasserstein loss function.} Several works utilize the Wasserstein distance as a training objective \citep{LWL,Boltzman}, and in GANs \citep{WGAN,IMGAN,RWGAN}. 
The Wasserstein distance
introduces a statistical estimator called the minimum Wasserstein estimator \citep{BASSETTI20061298}, which depends on the geometry of the data space. 
Recently, 
a Wasserstein ground metric was proposed~\citep{pmlr-v97-dukler19a}, which leads to a Wasserstein Lipschitz condition for the dual variable. 
In contrast to these works, here we apply the Wasserstein-2 distance to construct gradient operators for the optimization of GANs. This results in an iterative regularizer for the generator. 

\paragraph{Wasserstein gradient flows.} 
The Wasserstein-2 metric provides a metric tensor structure \citep{Lott,otto2001,LiG,Lafferty}. The gradient flow in the density manifold links with many transport-related partial differential equations \citep{vil2008,Nelson2}, such as the Fokker-Planck equation. 
There are two perspective: depending on approach of parametric \citep{LAND} or nonparametric models \citep{Stein}.
And in \cite{WGused} consider an approximate inference method for computing the Wasserstein gradient flow. Here an approximation of Kantorovich dual variables is introduced.  Compared to these works, we consider Wasserstein structure constrained on parameter space. There have been many approaches in this direction \citep{C2,GW,WM}. 
Compared to previous works, our approach applies the Wasserstein gradient to work on implicit generative models. 

\paragraph{Wasserstein proximal operator.} 

In full probability space with Wasserstein-2 distance, the proximal iteration or backward Euler method is known as the Jordan-Kinderlehrer-Otto (JKO) scheme \citep{JKO}. 
Many numerical methods have been proposed in this direction \citep{1809.10844}. 
We consider the backward and semi-backward Euler method on parameter space. 
Similar approaches have been considered by \cite{7558232}. We further approximate the Wasserstein proximal in affine function space. This yields a tractable iterative regularization term depending on statistics of the generators. 
Closely related to this article, 
~\cite{li2019affine} presented a proximal formulation of the Wasserstein natural gradient with the proximity term approximated over an affine subspace of functions in the Legendre dual formulation, but we focus on GANs.

\section{Discussion}
\label{sec:discussion}

We have developed approaches to practically implement the Wasserstein natural gradient method in the context of implicit deep generative models, which provide better minimizers, faster convergence in wall-clock time, and better stability. 
We consider a proximal method and obtain explicit formulas for the proximity term expressed in terms of statistics of the generated samples. 
Our method can be implemented at little to no additional cost over current methods. A novelty of our approach is that we regularize the generator, whereas much of the present work focuses on regularizing the discriminator. 

{
	Here we also notice that our methods brutally approximate the Wasserstein-2 metric and the gradient flows in generative models. To perform scientific computing of Wasserstein gradient flows, the efficacy of these approximations should be studied carefully. We leave these related data-driven, scientific-computing problems for future works. 
}

\subsubsection*{Acknowledgments} 
A. Lin, W. Li and S. Osher were supported by AFOSR MURI FA 9550-18-1-0502, AFOSR FA 9550-18-0167, ONR N00014-18-2527 and NSF DMS 1554564 (STROBE). G. Mont\'ufar has received funding from the European Research Council (ERC) under the European Union's Horizon 2020 research and innovation programme (grant agreement n\textsuperscript{o} 757983). 

\newpage 
\appendix

\section{Review of Wasserstein Information Geometry}

In this section, we briefly review the geometry of $L^2$-Wasserstein metric tensor in the probability set and probability models. For more details see \cite{LM}. 

Consider the set of probability densities with finite second moment, $\mathbb{P}_2(\mathbb{R}^n)$. 
Moreover, consider a metric function $W_2\colon \mathcal{P}_2(\mathbb{R}^n)\times \mathcal{P}_2(\mathbb{R}^n)\rightarrow \mathbb{R}_+$,
\begin{equation}
\begin{split}
W_2(\rho_0,\rho_1)^2=\inf_{\Phi_t}&\Big\{\int_0^1\int_{\mathbb{R}^n}\|\nabla\Phi(t,x)\|^2\rho(t,x)dxdt \colon \\
&\partial_t\rho(t,x)+\nabla\cdot(\rho(t,x)\nabla\Phi(t,x))=0, \\
&\rho(0,x)=\rho_0(x),~\rho(1,x)=\rho_1(x)\Big\},
\end{split}
\label{eq:W2}
\end{equation}
where the infimum is taken among all feasible Borel potential functions $\Phi\colon [0,1]\times \mathbb{R}^{n}\rightarrow \mathbb{R}$ and continuous density path $\rho\colon [0,1]\times \mathbb{R}^n\rightarrow \mathbb{R}_+$ satisfying the continuity equation. 
The variational formulation of (\ref{eq:W2}) introduces a Riemannian structure in density space. 
Consider the set of smooth and strictly positive probability densities
\begin{equation*}
\begin{split}
\mathcal{P}_+ &= \Big\{\rho \in C^{\infty}(\mathbb{R}^n)\colon \rho(x)>0,~\int_{\mathbb{R}^n}\rho(x)dx=1\Big\}  \subset \mathcal{P}_2(\mathbb{R}^n).
\end{split}
\end{equation*}
Writing $\mathcal{F}:=C^{\infty}(\mathbb{R}^n)$ for the set of smooth real valued functions, the tangent space of $\mathcal{P}_+$ is given by 
$$
T_\rho\mathcal{P}_+ = \Big\{\sigma\in \mathcal{F}\colon \int_{\mathbb{R}^n}\sigma(x) dx=0 \Big\}.
$$
Given $\Phi\in \mathcal{F}$ and $\rho\in \mathcal{P}_+$, define
\begin{equation*}
V_{\Phi}(x):=-\nabla\cdot(\rho(x) \nabla \Phi(x)).
\end{equation*}
Thus $V_\Phi\in T_{\rho}\mathcal{P}_+$. 
The elliptic operator $\nabla\cdot(\rho\nabla)$ identifies the function $\Phi$ modulo additive constants with the tangent vector $V_{\Phi}$ of the space of densities. 
Given $\rho\in \mathcal{P}_+$, $\sigma_i\in T_\rho\mathcal{P}_+$, $i=1,2$, define
\begin{equation*}
g^W_\rho(\sigma_1,\sigma_2)=\int_{\mathbb{R}^n}(\nabla\Phi_1(x), \nabla{\Phi_2}(x))\rho(x) dx,
\end{equation*}
where $\Phi_i(x)\in \mathcal{F}/\mathbb{R}$, such that $-\nabla\cdot(\rho \nabla\Phi_i)=\sigma_i$.  It we write $\Phi_i=-(\Delta_\rho)^{-1}\sigma_i$, then 
\begin{equation*}
g^W_\rho(\sigma_1,\sigma_2)=\int_{\mathbb{R}^n}\Big(\sigma_1(x), (-\Delta_\rho)^{-1} \sigma_2(x)\Big)dx.
\end{equation*}
The inner product $g^W$ endows $\mathcal{P}_+$ with a Riemannian metric tensor. In other words, the variational problem~\eqref{eq:W2} is a geometric action energy in $(\mathcal{P}_+, g^W)$. 

Given a loss function $F\colon \mathcal{P}_+\rightarrow \mathbb{R}$, the Wasserstein gradient operator in $(\mathcal{P}_+, g^W)$ is given by 
\begin{equation*}
\begin{split}
\textrm{grad}_W F(\rho)=&\Big((-\Delta_\rho)^{-1}\Big)^{-1}\frac{\delta}{\delta\rho(x)}F(\rho)\\
=&-\nabla\cdot(\rho\nabla \frac{\delta}{\delta\rho(x)}F(\rho)).
\end{split}
\end{equation*}
Thus the gradient flow satisfies \begin{equation*}
\frac{\partial\rho}{\partial t}=-\textrm{grad}_WF(\rho)= \nabla\cdot(\rho\nabla  \frac{\delta}{\delta\rho(x)} F(\rho)).
\end{equation*}
More analytical results on the Wasserstein-2 gradient flow have been discussed by \cite{Ambrosio2008Gradient}. 

We next consider Wasserstein-2 metric and gradient operator constrained on statistical models. 
A statistical model is defined by a triplet $(\Theta,  \mathbb{R}^n, \rho)$. 
For simplicity, let 
$\Theta\subset\mathbb{R}^d$ and let $\rho\colon \Theta\rightarrow \mathcal{P}( \mathbb{R}^n)$ be a parameterization function. 
We assume that the parameterization map $\rho$ is locally injective and satisfies suitable regularity conditions. 
We define a Riemannian metric $g$ on $\rho(\Theta)$ by pulling back the Wasserstein-2 metric tensor $g^W$. 
\begin{definition}[Wasserstein statistical manifold]\label{WSM}
	Given $\theta\in\Theta$ and $\dot\theta_i\in T_{\theta}\Theta$, $i=1,2$, we define 
	\begin{equation*}
	g_{\theta}(\dot\theta_1,\dot\theta_2)=\int_{\mathbb{R}^n} \Big((\dot\theta_1, \nabla_\theta\rho), (-\Delta_\rho)^{-1}(\nabla_\theta\rho, \dot\theta_2)\Big) dx. 
	\end{equation*}
	Equivalently, 
	\begin{equation*}
	g_{\theta}(\dot\theta_1,\dot\theta_2)=\int_{\mathbb{R}^n} \nabla \Phi_1(x)\nabla\Phi_2(x)\rho(\theta, x)dx,
	\end{equation*}
	where 
	\begin{equation*}
	-\nabla\cdot(\rho(\theta, x)\nabla\Phi_i(x))=(\nabla_\theta\rho(\theta ,x), \dot\theta_i).
	\end{equation*}
	Here $\nabla_\theta\rho=(\frac{\partial}{\partial\theta_i}\rho(\theta,x))_{i=1}^d\in\mathbb{R}^d$ and $(\cdot,\cdot)$ is an Euclidean inner product in $\mathbb{R}^d$.
\end{definition}
In particular, we have 
\begin{equation*}
g_\theta(\dot\theta_1, \dot\theta_2)=\dot\theta_1^\top G(\theta)\dot\theta_2,
\end{equation*}
where $G(\theta)=(G(\theta)_{ij})_{1\leq i,j\leq d}\in\mathbb{R}^{d\times d}$ is the associated metric tensor defined in Theorem \ref{prop2}. 
Thus the distance function can be written in terms of the geometry action functional
\begin{equation}\label{wmetric}
\begin{split}
& \textrm{Dist}(\theta, \theta^k)^2\\
= &\inf\Big\{\int_0^1 \dot\theta(t)^\top G(\theta(t))\dot\theta(t) dt\colon \theta(0)=\theta,~\theta(1)=\theta^k\Big\}\\
= & \inf \Big\{\int_0^1\int_{\mathbb{R}^n}(\partial_t\rho(\theta(t),x),\mathcal{G}(\rho_\theta)\partial_t\rho(\theta(t),x))dxdt\colon  \theta(0)=\theta,~\theta(1)=\theta^k\Big\}\\
= &\inf \Big\{\int_0^1\int_{\mathbb{R}^n}\|\nabla\Phi(t, x)\|^2\rho(\theta(t), x) dxdt \colon\\ 
& \hspace{1cm} \partial_t\rho(\theta(t),x)+\nabla\cdot(\rho(\theta(t),x)\nabla\Phi(t,x))=0, \theta(0)=\theta,~\theta(1)=\theta^k\Big\}.
\end{split}
\end{equation}



\section{Proofs of theorems}

\begin{proposition}[Semi-backward Euler method (Proposition \ref{prop3})]
	\label{app:prop3}
	The iteration 
	\begin{equation}\label{new}
	\theta^{k+1}=\arg\min_{\theta} F(\theta)+\frac{\tilde{D}(\theta, \theta^k)^2}{2h},
	\end{equation}
	with 
	\begin{equation*}
	\tilde{D}(\theta, \theta^k)^2=\int_{\mathbb{R}^n}\Big( \rho_\theta-\rho_{\theta^k} , \mathcal{G}(\rho_{\tilde\theta})(\rho_\theta-\rho_{\theta^k})\Big)dx,
	\end{equation*}
	and $\tilde\theta=\frac{\theta+\theta^k}{2}$, is a consistent time discretization of the Wassserstein natural gradient flow~(\ref{GD}).
\end{proposition}

\begin{proof}
	We claim that if $\|\theta-\theta^k\|=h$, then 
	\begin{equation}
	\label{claim0}
	(\theta^k-\theta)^\top G(\tilde\theta)(\theta^k-\theta)=\textrm{Dist}(\theta, \theta^k)^2+o(h^2),
	\end{equation}
	and
	\begin{equation}
	\label{claim}
	\begin{split}
	&\frac{1}{2}(\theta^k-\theta)^\top G(\tilde\theta)(\theta^k-\theta)+o(h^2)\\
	=& \sup_{\Phi}\int_{\mathbb{R}^n} \Phi(x)(\rho(\theta, x)-\rho(\theta^k, x)) - \frac{1}{2}\|\nabla\Phi(x)\|^2\rho(\tilde\theta,x) dx. 
	\end{split}
	\end{equation}
	We proceed with the proof of this claim. 
	Consider the geodesic path $\theta^\ast (t)$, $t\in[0,1]$, with $\theta^\ast (0)=\theta$, $\theta^\ast (1)=\theta^k$, s.t. 
	\begin{equation*}
	\textrm{Dist}(\theta, \theta^k)^2=\int_0^1 (\frac{d}{dt}\theta^\ast (t))^\top  G(\theta^\ast (t))\frac{d}{dt}\theta^\ast (t)dt.
	\end{equation*}
	We reparameterize the time of $\theta^\ast (t)$ into the time interval $[0, h]$. 
	Let $\tau=ht$ and $\theta(\tau)=\theta^\ast (ht)$. 
	Then $\theta(\tau)=\theta^k+\frac{\theta-\theta^k}{h} \tau+O(\tau^2)$ and 
	$\frac{d}{d\tau}\theta(\tau)=\frac{\theta-\theta^k}{h}+O(\tau)$, so that 
	\begin{equation*}
	\begin{split}
	\textrm{Dist}(\theta, \theta^k)^2 
	=&h\int_0^h \frac{d}{d\tau}\theta(\tau)^\top  G(\theta(\tau)) \frac{d}{d\tau}\theta(\tau) d\tau\\
	=&h\int_0^h (\frac{\theta-\theta^k}{h}+O(h))^\top G(\tilde\theta+O(h))(\frac{\theta-\theta^k}{h}+O(h)) d\tau \\
	=&(\theta-\theta^k)^\top G(\tilde\theta)(\theta-\theta^k)+o(h^2). 
	\end{split}
	\end{equation*}
	This proves equation~(\ref{claim0}). 
	We next prove equation~(\ref{claim}). 
	On the L.H.S.~of equation~(\ref{claim}), 
	\begin{equation*}
	\nabla_\theta\rho(\tilde\theta, x)(\theta-\theta^k)=\rho(\theta, x)-\rho(\theta^k,x)+o(h).
	\end{equation*}
	From the definition of $G(\theta)$, 
	\begin{align*}
	\frac{1}{2}(\theta-\theta^k)^\top G(\tilde\theta)(\theta-\theta^k) =\frac{1}{2}\int_{\mathbb{R}^n}\|\nabla\Phi(x)\|^2\rho(\tilde\theta, x)\,dx+o(h^2), 
	\end{align*}
	where 
	\begin{align*}
	-\nabla\cdot(\rho(\tilde\theta, x)\nabla\Phi(x)) &=\nabla_\theta\rho(\tilde\theta, x)(\theta-\theta^k).
	\end{align*}
	On the R.H.S.~of equation~(\ref{claim}), the maximizer $\Phi^\ast $ satisfies 
	\begin{equation}\label{c2}
	\rho(\theta, x)-\rho(\theta^k, x)+\nabla\cdot(\rho(\tilde\theta,x)\nabla\Phi^\ast (x))=0. 
	\end{equation}
	Inserting equation~(\ref{c2}) into the R.H.S.~of~(\ref{claim}), we obtain 
	\begin{equation*}
	\begin{split}
	&\int_{\mathbb{R}^n} \Phi^\ast (x)(\rho(\theta, x)-\rho(\tilde\theta, x)) -\frac{1}{2}\|\nabla\Phi^\ast (x)\|^2\rho(\tilde\theta,x)\, dx\\
	=&\int_{\mathbb{R}^n}\Phi^\ast (x)[-\nabla\cdot(\rho(\tilde\theta, x)\nabla\Phi^\ast (x)] - \frac{1}{2}\|\nabla\Phi^\ast (x)\|^2\rho(\tilde\theta,x)\, dx \\
	=&\int_{\mathbb{R}^n}\|\nabla\Phi^\ast (x)\|^2\rho(\tilde\theta, x)-\frac{1}{2}\|\nabla\Phi^\ast (x)\|^2\rho(\tilde\theta, x)\, dx\\
	=&\frac{1}{2}\int_{\mathbb{R}^n}\|\nabla\Phi^\ast (x)\|^2\rho(\tilde\theta, x)\, dx.
	\end{split}
	\end{equation*}
	Comparing the left and right hand sides of (\ref{claim}) yields the claim. 
	This allows us to write 
	\begin{equation*}
	\begin{split}
	\theta^{k+1}
	&=\arg\min_{\theta\in \Theta} F(\theta)+ \frac{1}{h} \frac{\textrm{Dist}(\theta, \theta^k)^2}{2}\\
	&=\arg\min_{\theta\in \Theta} F(\theta)+ \frac{1}{2h}\Big\{(\theta^k-\theta)^\top G(\tilde\theta)(\theta^k-\theta)+o(h^2)\Big\}\\
	&=\arg\min_{\theta\in \Theta}F(\theta)+\frac{1}{h}\Big\{\sup_{\Phi}\int_{\mathbb{R}^n} \Phi(x)(\rho(\theta, x)-\rho(\theta^k, x))\\
	&\hspace{3.5cm} -\frac{1}{2}\|\nabla\Phi(x)\|^2\rho(\tilde\theta,x) dx+o(h^2)\Big\}.
	\end{split}
	\end{equation*}
	We notice that 
	\begin{equation*}
	\begin{split}
	(\theta^k-\theta)^\top G(\tilde\theta)(\theta^k-\theta)+o(h^2)
	=\int_{\mathbb{R}^n}(\rho_{\theta^k}-\rho_{\theta}), \mathcal{G}(\rho_{\tilde\theta})(\rho_{\theta^k}-\rho_\theta) dx.     
	\end{split}
	\end{equation*}
	Thus we derive a consistent numerical method in time, known as the Semi-backward Euler method: \begin{equation*}
	\theta^{k+1}=\theta^k-h G(\tilde\theta)^{-1}\nabla_\theta F(\theta^{k+1})+o(h).
	\end{equation*}
	\qed
\end{proof}


\begin{theorem}[Affine metric function $\tilde D$ (Theorem \ref{thm4})]
	Consider some $\Psi=(\psi_1,\ldots,\psi_K)^{\top}$ and 
	assume that $M(\theta)=(M_{ij}(\theta))_{1\leq i,j\leq K}\in \mathbb{R}^{K\times K}$ is a regular matrix with entries
	\begin{equation*}
	M_{ij}(\theta)=\mathbb{E}_{Z\sim p}\Big(\sum_{l=1}^n\partial_{x_l}\psi_i(g(\tilde\theta, Z)) \partial_{x_l}\psi_j(g(\tilde\theta, Z))  \Big),  
	\end{equation*}
	where $\tilde \theta=\frac{\theta+\theta^k}{2}$. 
	Then, 
	\begin{equation*}
	\begin{split}
	\tilde{D}(\theta, \theta^k)^2 = &\Big(\mathbb{E}_{Z\sim p}[\Psi(g(\theta, Z))-\Psi(g(\theta^k, Z))]\Big)^\top \\
	&M(\tilde \theta)^{-1}\Big(\mathbb{E}_{Z\sim p}[\Psi(g(\theta, Z))-\Psi(g(\theta^k, Z))]\Big). 
	\end{split}
	\end{equation*}
\end{theorem}
\begin{proof}
	The gradient of our function approximator w.r.t. the input space variable is 
	\begin{equation*}
	\nabla \Phi= (\sum_j \xi_j \partial_i\psi_j(x))_{i=1}^n. 
	\end{equation*}
	The squared norm of the gradient is 
	\begin{equation*}
	\begin{split}
	\| \nabla \Phi\|^2 
	=& \sum_i (\sum_j \xi_j \partial_i \psi_j(x))^2 
	= \sum_i \sum_j \xi_j \partial_i \psi_j \sum_k \xi_k \partial_i \psi_k \\
	=& \sum_j \sum_k \xi_j \xi_k (\sum_{i}\partial_i\psi_j(x)\partial_i\psi_k(x))\\
	=& \xi^\top C(x) \xi, 
	\end{split}
	\end{equation*}
	where $C_{ij}(x) = \sum_k\partial_k\psi_i\partial_k\psi_j$. 
	Now consider the distance 
	\begin{align*}
	\tilde{D}(\theta, \theta^k)^2
	=& \sup_{\Phi \in\mathcal{F}_\xi} \int_{\mathbb{R}^n} \Phi(\rho_\theta -\rho_{\theta^k})dx - \frac{1}{2} \int_{\mathbb{R}^n} (\nabla \Phi)^2 \rho_{\tilde \theta} dx \\
	=& \sup_{\xi} \xi^\top (\mathbb{E}_\theta \psi - \mathbb{E}_{\theta^k}\psi)  -\frac{1}{2} \xi^\top \mathbb{E}_{\tilde\theta} C \xi.
	\end{align*}
	Here $\mathbb{E}_{\tilde\theta}C$ is a positive semi-definite matrix. 
	Since for any $\xi\in \mathbb{R}^K$, we have
	\begin{equation*}
	\xi^\top \mathbb{E}_{\tilde\theta}C\xi=\int_{\mathbb{R}^n}\sum_i (\sum_j \xi_j \partial_i \psi_j(x))^2 \rho_{\theta}dx\geq 0.    
	\end{equation*}
	Under the assumption that $\mathbb{E}_\theta C$ is invertible, the optimization problem is strictly concave. At the maximizer, we have 
	\begin{equation*}
	\xi^\ast =(\mathbb{E}_{\tilde\theta} C)^{-1} (\mathbb{E}_\theta \psi -\mathbb{E}_{\theta^k} \psi).
	\end{equation*}
	Thus, 
	\begin{equation*}
	\tilde{D}(\theta, \theta^k)^2= (\mathbb{E}_\theta \psi -\mathbb{E}_{\theta^k} \psi)^\top (\mathbb{E}_{\tilde\theta} C)^{-1} (\mathbb{E}_\theta \psi -\mathbb{E}_{\theta^k} \psi),
	\end{equation*}
	which completes the proof. 
	\qed
\end{proof}


\begin{theorem}[1-D sample space (Theorem \ref{1d})]
	If $n=1$, then
	\begin{multline*}
	\operatorname{Dist}(\theta_0, \theta_1)^2 = \inf \Big\{\int_0^1\mathbb{E}_{Z\sim p}\|\frac{d}{dt}g(\theta(t), Z)\|^2\, dt \colon \theta(0)=\theta_0, \theta(1)=\theta_1 \Big\}, 
	\end{multline*}
	where the infimum is taken over all continuously differentiable parameter paths. 
	Therefore, we have 
	\begin{equation*}
	\tilde{D}(\theta, \theta^k)^2=\mathbb{E}_{Z\sim p}\|g(\theta, Z)-g(\theta^k, Z)\|^2.
	\end{equation*}
\end{theorem} 

\begin{proof}
	The implicit model is given by a push-forward relation $g_\theta \# p(z)=\rho(\theta, x)$, so that  
	\begin{align*}\label{a}
	\int_{\mathbb{R}^m} f(g(\theta, z))p(z)dz=\int_{\mathbb{R}^n} f(x)\rho(\theta,x)dx, 
	\end{align*}
	for any $f\in C_c^{\infty}(\mathbb{R}^n)$. 
	If $f\in C^{\infty}_c({\mathbb{R}^n})$, then
	\begin{equation}\label{form1}
	\begin{split}
	\frac{d}{dt}\mathbb{E}_{Z\sim p(z)} f(g(\theta(t),Z))
	=&\frac{d}{dt}\int_{\mathbb{R}^m} f(g(\theta(t),z))p(z)dz\\
	=&\frac{d}{dt}\int_{\mathbb{R}^n} f(x)\rho(\theta(t),x)dx\\
	=&\int_{\mathbb{R}^n} f(x)\frac{\partial}{\partial t}\rho(\theta(t),x)dx\\
	=&\int_{\mathbb{R}^n} f(x)(-\nabla\cdot(\rho(\theta(t),x)\nabla\Phi(t,x)))dx\\
	=&\int_{\mathbb{R}^n}\nabla f(x) \nabla\Phi(t,x)\rho(\theta(t),x)dx\\
	=&\int \nabla f(g(\theta, z)) \nabla\Phi(t,g(\theta, z)) p(z)dz, 
	\end{split}
	\end{equation}
	where the last equality holds from the push forward relation. 
	On the other hand, 
	\begin{equation}\label{form2}
	\begin{split}
	\frac{d}{dt}\mathbb{E}_{Z\sim p(z)} f(g(\theta(t),Z)) 
	&=\lim_{\Delta t\rightarrow 0}\mathbb{E}_{Z\sim p(z)}\frac{f(g(\theta(t+\Delta t), Z)-f(g(\theta(t),Z))}{\Delta t}\\
	&=\lim_{\Delta t\rightarrow 0}\int_{\mathbb{R}^m}\frac{f(g(\theta(t+\Delta t), z))-f(g(\theta(t), z))}{\Delta t}p(z)dz\\
	&=\int_{\mathbb{R}^m} \nabla f(g(\theta(t),z)) \frac{d}{dt}g(\theta(t),z) p(z)dz, 
	\end{split}
	\end{equation}
	where $\nabla$ and $\nabla\cdot$ are the gradient and divergence operators w.r.t.~$x\in{\mathbb{R}^n}$. 
	The second last equality holds from the push forward relation, and the last equality holds using integration by parts w.r.t.~$x$. 
	Since (\ref{form1}) equals (\ref{form2}) for any $f\in C_c^{\infty}({\mathbb{R}^n})$, we have 
	\begin{equation*}
	\int \nabla f(g(\theta, z)) \nabla\Phi(t,g(\theta, z)) p(z)\,dz = \int \nabla f(g(\theta(t),z)) \frac{d}{dt}g(\theta(t),z) p(z)\,dz.
	\end{equation*}
	Thus, 
	\begin{equation*}
	\begin{split}
	\int \nabla f(g(\theta, z)) \Big(\nabla\Phi(t,g(\theta, z))-\frac{d}{dt}g(\theta(t),z)\Big) p(z)dz=0.
	\end{split}
	\end{equation*}
	If $n=1$, then $\nabla f$ can be any function in $\mathbb{R}^1$. 
	For each $t$, choosing $\nabla f(g(\theta,z))=\nabla\Phi(t, g(\theta,z))-\frac{d}{dt}g(\theta(t),z)$, we obtain 
	\begin{equation*}
	\begin{split}
	\int |\nabla\Phi(t,g(\theta, z))-\frac{d}{dt}g(\theta(t),z)|^2 p(z)dz=0.
	\end{split}
	\end{equation*}
	Hence, 
	\begin{equation*}
	\frac{d}{dt}g(\theta(t),z)=\nabla\Phi(t,g(\theta(t),z)).    
	\end{equation*}
	In turn, by the definition of the push forward operation, we have 
	\begin{align*}
	\mathbb{E}_{Z\sim p(z)} \|\frac{d}{dt}g(\theta(t),Z)\|^2 
	=&\int_{\mathbb{R}^n} \|\nabla\Phi(t, g(\theta(t),z))\|^2p(z)\,dz\\
	=&\int_{\mathbb{R}^n} \|\nabla\Phi(t,x)\|^2\rho(\theta(t),x)\,dx,
	\end{align*}
	which finishes the proof.
	\qed
\end{proof}



\section{A practical description of the Wasserstein proximal}
\label{sec:rwp_detailed}

As mentioned in Section~\ref{resultsrwp}, the Relaxed Wasserstein Proximal is meant to be an easy-to-implement, drop-in regularization. 
For instructional purposes, we take a specific example to showcase the algorithm: Relaxed Wasserstein Proximal on Vanilla GANs (with non-saturating gradient for the generator):

\begin{itemize}
	\item  Given:
	\begin{itemize}
		\item A generator $g_\theta$, and discriminator $D_\omega$,
		\item The distance function $F_\omega(g_\theta) = \mathbb{E}_{x\sim\text{real}}[\log (D_\omega(x))] - \mathbb{E}_{z\sim \mathcal{N}(0,1)}[\log (1 - D_\omega(g_\theta(z))]$,
		\item Choice of optimizers, $\text{Adam}_{\omega}$ and $\text{Adam}_{\theta}$,
		\item Proximal step-sizes $h$, and generator iterations $\ell$, and
		\item Batch size $B$.
	\end{itemize}
	
	\item  Then the algorithm follows:
	\begin{enumerate}
		\item Sample real data $\{x_i\}_{i=1}^B$, and latent data $\{z_i\}_{i=1}^B$.
		\item Update the discriminator:
		\begin{flalign*}
		&\omega^{k}\leftarrow \text{Adam}_{\omega}\bigg(-\frac{1}{B} \sum_{i=1}^B \log(D_\omega(x_i)) - \frac{1}{B} \sum_{i=1}^B \log(1 - D_\omega(g_\theta(z_i)))\bigg)
		\end{flalign*}
		\item Sample latent data $\{z_i\}_{i=1}^B$
		\item 
		Perform Adam gradient descent $\ell$ times on the generator:
		\begin{flalign*}
		&\theta^{k} \leftarrow \text{Adam}_\theta \bigg(-\frac{1}{B} \sum_{i=1}^B \log(D_\omega(g_\theta(z_i))) - \frac{1}{B} \sum_{i=1}^B \frac{1}{2h} \|g_\theta(z_i) - g_{\theta^{k-1}}(z_i)\|_2^2 \bigg), \\
		& \qquad \qquad \text{ for $\ell$ number of times.}
		\end{flalign*}
		\item 
		Repeat steps 1--4 until a stopping criterion is met (e.g., maximum number of iterations). 
	\end{enumerate}
\end{itemize}

As we can see from the above description, the only difference between the standard way of training GANs and using the Relaxed Wasserstein Proximal, are the $\|g_\theta(z_i) - g_{\theta^{k-1}}(z_i)\|_2^2$ terms and the number of generator iterations $\ell$. 
Note that in this paper, we call a single loop of updating the discriminator once and then updating the generator a number of a times, an outer-iteration.

\section{Details on the experiments}
\label{sec:hp_for_rwp}

The hyperparameter settings for the RWP, Order-1 SBE, and Order-2 Diagonal SBE experiments in Section~\ref{section4} are the following: 

\begin{itemize}
	\item A mini-batch size of 64 for all experiments. 
	\item For CIFAR-10 with WGAN-GP: The Adam optimizer with learning rate $0.0001$,  $\beta_1=0.5$, and $\beta_2=0.9$ for both the generator and discriminator. 
	We used a latent space dimension of $128$. For RWP, we used $h=0.1$, and $\ell=10$ generator iterations. 
	For Order-1 SBE, we used $h=0.5$, and $\ell=5$. 
	For Order-2 Diagonal SBE, we used $h=0.2$ and $\ell=5$. 
	\item For CIFAR-10 with Vanilla and DRAGAN: The Adam optimizer with learning rate $0.0002$, $\beta_1=0.1$, and $\beta_2=0.999$ for both the generator and discriminator. 
	We used a latent space dimension of $100$. 
	For RWP, we used $h=0.2$, and $\ell=5$ generator iterations. 
	For Order-1 SBE, we used $h=0.2$ and $\ell=5$. 
	For Order-2 Diagonal SBE, we used $h=0.2$ and $\ell=5$. 
	\item For aligned and cropped CelebA with Vanilla: The Adam optimizer with learning rate $0.0002$, $\beta_1=0.5$, and $\beta_2=0.999$ for both the generator and discriminator. 
	We used a latent space dimension of $100$. 
	For RWP, we used $h=0.2$, and $\ell=5$ generator iterations. 
	For Order-1 SBE, we used $h=0.2$ and $\ell=5$. 
	For Order-2 Diagonal SBE, we used $h=0.2$ and $\ell=5$. 
	\item For aligned and cropped CelebA with WGAN-GP: The Adam optimizer with learning rate $0.0001$, $\beta_1=0.5$, and $\beta_2=0.9$ for both the generator and discriminator. 
	We used a latent space dimension of $128$. 
	For RWP, we used $h=0.1$, and $\ell=10$ generator iterations. 
	For Order-1 SBE, we used $h=0.5$ and $\ell=5$, but we raised the number of discriminator iterations to $7$ (as opposed to the usual $5$. 
	For Order-2 Diagonal SBE, we used $h=0.2$ and $\ell=5$. 
	\item For the high-learning rate for CelebA with WGAN-GP: The hyperparameters are the same as WGAN-GP except in the following: 
	the learning rate is raised to $0.002$, for RWP we have $h=0.1$ and $\ell=5$, 
	for Order-1 SBE we have $h=0.05$ and $\ell=5$, 
	for Order-2 Diagonal SBE we have $h=0.05$ and $\ell=3$. 
	\item For the high Adam $\beta_1$ momentum for CelebA with WGAN-GP: The hyperparameters are the same as WGAN-GP except in the following: 
	the $\beta_1$ parameter is raised to $0.5$ (as opposed to $0$), 
	for RWP we have $h=0.1$ and $\ell=10$, 
	for Order-1 SBE we have $h=0.05$ and $\ell=5$, 
	for Order-2 Diagonal SBE we have $h=0.05$ and $\ell=3$. 
\end{itemize}

\section{More figures}

\begin{figure}[H]
	\centering
	\includegraphics[clip=true, trim=1cm 0.75cm 0cm 0cm,
	width=.8\linewidth]{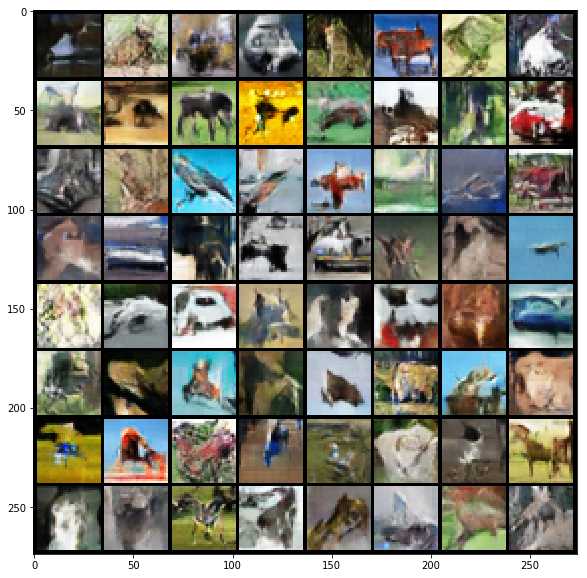}
	\caption{A sample of images generated by a neural network that was trained  on the CIFAR-10 dataset using the WGAN-GP framework with RWP regularization.}
	\label{app:fig:64_sample_cfiar10_wgangp_rwp}
\end{figure}

\begin{figure}[H]
	\centering
	\includegraphics[clip=true, trim=1cm 0.75cm 0cm 0cm,
	width=.8\linewidth]{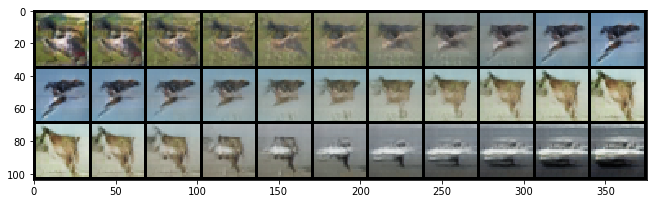}
	\caption{A latent space walk for a network trained using the WGAN-GP framework with RWP regularization on CIFAR-10. 
		The latent space walk is interpolating between 4 points in latent space. 
		The smooth transitions indicate good generalization. 
	}
	\label{app:fig:cifar10_wgangp_rwp}
\end{figure}

\begin{figure}[H]
	\centering
	\includegraphics[clip=true, trim=1cm 0.75cm 0cm 0cm,
	width=.8\linewidth]{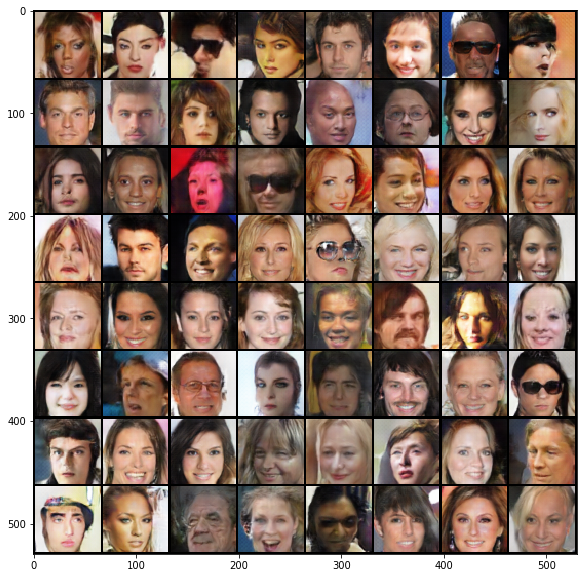}
	\caption{A sample of images generated by a neural network that was trained on the CelebA dataset using the Vanilla GAN framework with RWP regularization. }
	\label{app:fig:64_sample_celeba}
\end{figure}

\begin{figure}[H]
	\centering
	\includegraphics[clip=true, trim=1cm 0.75cm 0cm 0cm,
	width=.8\linewidth]{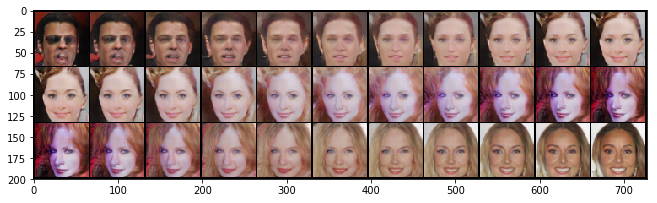}
	\caption{A latent space walk for an implicit generative model that was trained in the Vanilla GAN framework using RWP regularization on the CelebA dataset. As we have smooth transitions, this shows the generator is not overfitting. The latent space walk is done by interpolating between 4 points in the latent space.}
	\label{app:fig:latent_space_walk}
\end{figure}



\section{A three neural-network version}

\begin{algorithm}[H]
	\caption{Semi-backward Euler method, where $F_\omega$ is a parameterized function to minimize.}
	\label{alg:wassprox}
	\begin{algorithmic}[1]
		\REQUIRE $F_\omega$, a parameterized function to minimize (e.g. Wasserstein-1 with a parameterized discriminator). $g_\theta$ the generator. $\Phi_p$ the potential.
		\REQUIRE $h$ the proximal step-size, $m$ the batch size.
		\REQUIRE $\text{Optimizer}_{{F}_\omega}$, $\text{Optimizer}_{g_\theta}$, and $\text{Optimizer}_{\Phi_p}$
		\REQUIRE The number of generator iterations and p iterations to do per update.
		\FOR{$k = 0$ \TO max iterations}
		\STATE{Sample real data $\{x_i\}_{i=1}^B$ and latent data $\{z_i\}_{i=1}^B$.}
		\STATE{$\omega^k \leftarrow \text{Optimizer}_{{F}_\omega} \left(\frac{1}{B} \sum_{i=1}^B {F}_\omega(g_\theta(z_i))\right)$}
		\FOR{$s = 0$ \TO max iterations for phi}
		\STATE{Sample latent data $\{z_i\}_{i=1}^B$}
		\STATE{$p^k \leftarrow \text{Optimizer}_{\Phi_p} \left( \frac{1}{h}\frac{1}{B}  \sum_{i=1}^B \Phi_p(g_\theta(z_i)) - \Phi_p(g_{\theta^{k-1}}(z_i)) - \frac{1}{2} \|\nabla \Phi_p(g_{\theta^{k-1}}(z_i))\|^2 \right)$}
		\ENDFOR
		\FOR{$\ell=0$ \TO max iterations for generator}
		\STATE{Sample latent data $\{z_i\}_{i=1}^B$}
		\STATE{$\theta^k \leftarrow \text{Optimizer}_{g_\theta} \left( \frac{1}{B} \sum_{i=1}^B  {F}_\omega(g_\theta(z_i)) + \right.$}
		\STATE{\hspace{3cm}$\left.\frac{1}{h}\left(\Phi_p(g_\theta(z_i)) - \Phi_p(g_{\theta^{k-1}}(z_i)) - \frac{1}{2} \|\nabla \Phi_p(g_{\theta^{k-1}}(z_i))\|^2\right)\right)$}
		\ENDFOR
		\ENDFOR
	\end{algorithmic}
\end{algorithm}

\bibliographystyle{plain}
\bibliography{WNG}

\end{document}